\documentclass[10pt,a4paper,twoside]{article}
\usepackage{babel}
\usepackage[latin1]{inputenc}
\usepackage[T1]{fontenc}
\usepackage{amsmath,amssymb,amsthm}
\usepackage{geometry}
\usepackage{booktabs}
\usepackage{array}
\usepackage{multirow}
\usepackage[all]{xy}
\usepackage{hyperref}
\usepackage{url}
\usepackage[numbers,sort&compress]{natbib}
\usepackage{algorithm}
\usepackage{algorithmic}
\usepackage{fancyhdr}
\usepackage{verbatim}
\usepackage{lastpage}
\usepackage{ifthen}

\usepackage[pdftex]{color,graphicx}
\usepackage{hyperref}

\usepackage{iberamia}   

\newtheorem{theorem}{Theorem}
\newtheorem{proposition}[theorem]{Proposition}
\newtheorem{corollary}[theorem]{Corollary}
\newtheorem{definition}[theorem]{Definition}
\newtheorem{remark}[theorem]{Remark}

\newcommand{\tnorm}{\otimes}      
\newcommand{\resid}{\Rightarrow}  
\newcommand{\disj}{\oplus}        
\newcommand{\LNN}{\L{}NN}

\begin{document}

\title{\L{}ukasiewicz Neural Networks Extended: Residual Architectures and
       Crystallization Strategies for Interpretable Rule Extraction}

\author{Carlos Leandro\\
  \textit{Departamento de Matem\'atica,
  Instituto Superior de Engenharia de Lisboa (ISEL)},\\
  \textit{Instituto Polit\'ecnico de Lisboa, Portugal}\\
  \texttt{miguel.melro.leandro@gmail.pt}}

\date{}
\maketitle

\begin{abstract}
A feed-forward neural network whose weights are integers and whose activation
is the truncated identity implements, neuron by neuron, the connectives of
\L{}ukasiewicz many-valued logic.
This exact correspondence --- established theoretically by Castro and Trillas
and developed into a training algorithm by Leandro --- enables \emph{symbolic
knowledge extraction}: training produces not a black-box model but a logical
formula.
Two obstacles have limited the approach to shallow architectures and small
datasets: crystallization (forcing weights to integers) succeeds only
probabilistically under the original Levenberg--Marquardt training scheme,
and the theoretical guarantees break down as networks grow deeper.

This paper addresses both obstacles.
First, we prove that \emph{residual connections} (skip connections of the kind
used in ResNets) extend \L{}ukasiewicz neural networks to arbitrary depth while
preserving the symbolic correspondence \emph{at merge neurons} by
construction: merge neurons in a \L{}ukasiewicz residual block automatically
satisfy the neuron-classification proposition, regardless of the inner layer
weights; inner-layer neurons are trained toward representability by the
crystallization strategy.
Second, we analyse three crystallization strategies --- Levenberg--Marquardt
(corrected), straight-through estimation (STE), and proximal regularization ---
characterizing their theoretical guarantees, failure modes, and
interpretability trade-offs.
Experiments on six UCI benchmarks show that no single strategy dominates:
LM with residual connections converges in under ten iterations on structured
problems and recovers exact logical formulas, including
$F = \neg(\mathit{body\_shape} \oplus \mathit{jacket\_colour})$ on MONK-3 at
$97.2\%$ accuracy; STE provides the most reliable crystallization (100\% rate
across all datasets) and the best aggregate accuracy; proximal regularization
crystallizes consistently but collapses on dense feature sets, while
identifying clinically meaningful features on the Heart Disease benchmark.
The resulting framework is the first to combine provably interpretable residual
architectures with a systematic analysis of integer-weight training.

\end{abstract}

\section{Introduction}
\label{sec:intro}

The tension between predictive power and interpretability in machine learning
has deepened as models have grown larger.
A neural network that achieves human-level performance on a medical diagnosis
task may be clinically useless if its reasoning cannot be audited.
Rudin~\cite{rudin2019} argues that in high-stakes domains the correct response
is not to add a post-hoc explanation layer on top of a black-box model, but to
require that the model itself be interpretable by construction.
Post-hoc explanation methods provide approximations --- and an approximation
to a decision can be systematically wrong in ways that are invisible to the
explainer~\cite{guidotti2018}.

Symbolic rule extraction from neural networks has pursued this goal since the
1980s~\cite{gallant1988,towell1993,towell1994}.
The dominant paradigm trains a network, then searches for a logical rule that
approximates its input-output behaviour.
The approximation step is unavoidable in classical two-valued logic: a neuron
computing a weighted sum is not, in general, a Boolean connective.
\L{}ukasiewicz logic removes this obstacle.
Its connectives --- conjunction, implication, negation, and disjunction ---
are all piecewise linear functions on $[0,1]$, the same domain that
truncated-identity neurons compute over.
Castro and Trillas~\cite{castro1998} observed that this structural coincidence
is not accidental: a neuron with integer weights $\{-1,0,1\}$ and truncated
identity activation \emph{is} a \L{}ukasiewicz connective, exactly, without
approximation.
Leandro~\cite{leandro2009} turned this observation into a complete
methodology: train a feed-forward network so that its weights crystallize to
integers, then read the resulting formula directly from the network topology.

The programme is mathematically clean and practically motivated, but two
limitations have constrained its reach.
The Levenberg--Marquardt (LM) training algorithm used in the original work
crystallizes weights with high probability on small, low-dimensional datasets,
but fails systematically on high-dimensional real-world problems: a
conditioning bug in the original implementation causes weight updates of
$\|\Delta w\| \approx 47$ when $\|w\| \approx 0.1$, diverging before
convergence.
More fundamentally, the symbolic guarantees are established for two-layer
feed-forward architectures; it is not obvious that they extend to deeper
networks without imposing additional constraints on every weight in the network.

This paper extends the \L{}ukasiewicz neural network framework along both
dimensions.
On the theoretical side, we show that \emph{residual connections}
--- the skip-connection mechanism introduced by He~et~al.~\cite{he2016} for
deep image recognition networks --- can be incorporated into \L{}ukasiewicz
neural networks (\LNN{}s) with the symbolic guarantee automatically preserved.
The key insight is structural: in a \L{}ukasiewicz residual block, the merge
neuron that combines the residual path and the shortcut receives exactly two
inputs with weights $+1$, placing it in the configuration covered by the
neuron-classification proposition regardless of what the inner layers compute.
On the practical side, we analyse two strategies that guarantee crystallization
by construction --- straight-through estimation (STE)~\cite{bengio2013} and
proximal regularization --- alongside the corrected LM, and characterize the
trade-offs between them.

\paragraph{Contributions.}
\begin{enumerate}
  \item \textbf{Residual \L{}ukasiewicz networks (theory).}
    We prove that merge neurons in a \L{}ukasiewicz residual block satisfy
    Proposition~\ref{prop:classif} by construction.
    This establishes that residual \LNN{}s are a theoretically sound extension
    of the original framework to deeper architectures.
  \item \textbf{Crystallization analysis.}
    We identify and correct a critical bug in the original LM training procedure,
    characterize the theoretical guarantees of three crystallization strategies,
    and establish the conditions under which each strategy succeeds or fails.
  \item \textbf{Empirical evaluation.}
    We compare the three strategies across six UCI benchmarks (10 independent
    trials per dataset), providing statistical evidence (Wilcoxon signed-rank
    tests) for differences in accuracy and crystallization rate.
  \item \textbf{Formula extraction in practice.}
    We demonstrate extraction of exact, fully representable formulas ---
    including a clinically meaningful five-variable rule for heart disease
    diagnosis --- and show that the Proximal regularizer reliably identifies
    the strongest clinical predictor (\texttt{ca}, the number of stenosed
    vessels) across all restarts.
\end{enumerate}

The remainder of the paper is organized as follows.
Section~\ref{sec:background} reviews \L{}ukasiewicz logic and the original
\LNN{} framework.
Section~\ref{sec:residual} introduces residual \LNN{}s and proves the merge-neuron
theorem.
Section~\ref{sec:crystal} analyses the three crystallization strategies.
Section~\ref{sec:experiments} presents the experimental evaluation.
Section~\ref{sec:related} positions the work relative to the literature.
Sections~\ref{sec:discussion} and~\ref{sec:conclusion} discuss implications and
conclude.

\section{Background}
\label{sec:background}

\subsection{\L{}ukasiewicz many-valued logic}
\label{ssec:logic}

In classical propositional logic, truth values are drawn from $\{0,1\}$.
\L{}ukasiewicz logic generalizes this to the real unit interval $[0,1]$,
interpreting $0$ as complete falsity and $1$ as complete truth, with
intermediate values representing degrees of certainty~\cite{hajek1998}.
The four fundamental connectives are defined as follows.
\begin{definition}[\L{}ukasiewicz connectives]
For $x, y \in [0,1]$:
\begin{align}
  x \tnorm y  &= \max(0,\, x + y - 1),   \label{eq:tnorm}\\
  x \resid y  &= \min(1,\, 1 - x + y),   \label{eq:resid}\\
  \neg x       &= 1 - x,                  \label{eq:neg}\\
  x \disj y   &= \min(1,\, x + y).       \label{eq:disj}
\end{align}
\end{definition}
The operations $\tnorm$ and $\disj$ recover Boolean conjunction and disjunction
when restricted to $\{0,1\}$; they are, however, genuinely many-valued on $(0,1)$.
For instance, $\tfrac{1}{2} \tnorm \tfrac{1}{2} = 0$: two half-truths do not
compose into a truth under conjunction unless their sum exceeds $1$.
Note that $x \disj y = \neg(\neg x \tnorm \neg y)$, so disjunction is
De~Morgan dual to conjunction, as in classical logic.
The residuum $\resid$ is the natural implication associated with the t-norm
$\tnorm$~\cite{hajek1998}.

A \emph{McNaughton function}~\cite{mcnaughton1951} is a continuous, piecewise
linear map $f\colon [0,1]^n \to [0,1]$ with integer coefficients in each linear
piece.
McNaughton's theorem states that $f$ is a McNaughton function if and only if
it is expressible as a formula in infinitely-valued \L{}ukasiewicz logic.
All four connectives in Eqs.~\eqref{eq:tnorm}--\eqref{eq:disj} are
piecewise linear with integer coefficients, confirming that the logic is
exactly characterised by this function class~\cite{mundici1986}.

For any integer $n \geq 1$, the finite set
$S_n = \{0, 1/n, 2/n, \ldots, 1\}$ is closed under all four operations.
These finite subdomains allow truth-table-style evaluation and make the
semantics computationally tractable.

\subsection{\L{}ukasiewicz neural networks}
\label{ssec:lnn}

A neuron with $m$ inputs $x_1, \ldots, x_m \in [0,1]$, weights
$w_1, \ldots, w_m \in \mathbb{R}$, bias $b \in \mathbb{R}$, and
\emph{truncated identity activation}
\begin{equation}
  \psi(x) = \min\!\bigl(1,\, \max(0,\, x)\bigr),
  \label{eq:psi}
\end{equation}
produces output $z = \psi\!\bigl(\sum_{i=1}^m w_i x_i + b\bigr)$.
This activation is piecewise linear with integer coefficients, matching the
structure of McNaughton functions.

The following proposition, established by Castro and
Trillas~\cite{castro1998} and extended by Leandro~\cite{leandro2009},
is the central result connecting neural computation to \L{}ukasiewicz logic.

\begin{proposition}[Neuron classification~\cite{castro1998,leandro2009}]
\label{prop:classif}
Let $\alpha = \psi_b(-x_1, \ldots, -x_n,\, x_{n+1}, \ldots, x_m)$
be a neuron with $n$ weights equal to $-1$ and $p = m - n$ weights equal to
$+1$.
\begin{enumerate}
  \item If $b = -p + 1$, then $\alpha$ computes the \emph{conjunction}
    $\neg x_1 \tnorm \cdots \tnorm \neg x_n \tnorm x_{n+1}
     \tnorm \cdots \tnorm x_m$.
  \item If $b = n$, then $\alpha$ computes the \emph{disjunction}
    $\neg x_1 \disj \cdots \disj \neg x_n \disj x_{n+1}
     \disj \cdots \disj x_m$.
\end{enumerate}
\end{proposition}

\begin{proof}[Proof sketch]
For case 1 with $n=0$: $\psi_{-p+1}(x_1,\ldots,x_m)
= \max(0, x_1+\cdots+x_m - (p-1)) = x_1 \tnorm \cdots \tnorm x_m$,
which follows by induction using $x \tnorm y = \max(0,x+y-1)$.
The general case with $n > 0$ follows by substituting $\neg x_i = 1 - x_i$
and collecting terms.
Case 2 is analogous.
\end{proof}

A network in which \emph{every} neuron satisfies Proposition~\ref{prop:classif}
is a \emph{\L{}ukasiewicz neural network} (\LNN{}).
In a \LNN{}, the output can be read as a \L{}ukasiewicz formula by composing
the sub-formulas computed at each neuron, propagating bottom-up through the
layers.
Conversely, any \L{}ukasiewicz formula can be encoded as a \LNN{} by mapping
each connective to a neuron with the corresponding bias and weights
$\{-1, 0, +1\}$~\cite{amato2002}.
This bidirectional correspondence --- formula $\leftrightarrow$ network ---
is the foundation of the symbolic extraction methodology.

\subsection{Rewriting algebra and symbolic extraction}
\label{ssec:rewriting}

Proposition~\ref{prop:classif} covers neurons with at most two distinct weight
values ($-1$ and $+1$).
Real networks, after training, may have neurons with more complex weight
configurations that do not directly satisfy the proposition.
The \emph{rewriting algebra} of~\cite{leandro2009} handles this case.

\paragraph{Rule R.}
The following transformation decomposes a neuron with $n$ inputs into a tree
of neurons with two inputs:
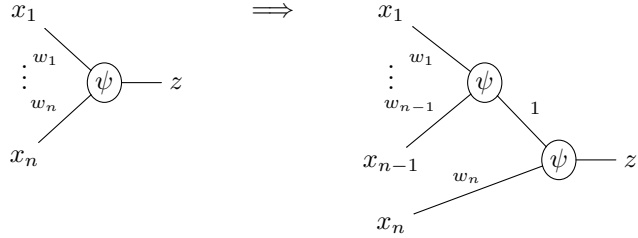
\[
\xymatrix @R=10pt @C=15pt {
  x_1 \ar@{-}[dr]_{w_1} & \\
  \vdots & *+[o][F-]{\psi} \ar@{-}[r] & z \\
  x_n \ar@{-}[ur]^{w_n} &
}
\qquad \Longrightarrow \qquad
\xymatrix @R=10pt @C=15pt {
  x_1 \ar@{-}[dr]_{w_1} & \\
  \vdots & *+[o][F-]{\psi} \ar@{-}[rd]^{1} & \\
  x_{n-1} \ar@{-}[ur]^{w_{n-1}} & & *+[o][F-]{\psi} \ar@{-}[r] & z \\
  x_n \ar@{-}[rru]^{w_n} & &
}
\]
The split is valid when the bias conserves ($b = b_0 + b_1$) and neither
sub-neuron produces a constant output.
Applying Rule~R to all possible splittings of a multi-input neuron generates
a set of candidate two-neuron networks.

\paragraph{$\lambda$-similarity.}
The proximity between two networks over a finite subdomain $S_n$ is measured by
\begin{equation}
  \lambda = \exp\bigl(-\mathrm{MAE}(f(X),\, y)\bigr),
  \label{eq:lambda}
\end{equation}
where $\mathrm{MAE}$ is the mean absolute error over all input vectors $X \in S_n^m$.
When $\lambda = 1$, the two functions are identical on $S_n$; as $\lambda \to 0$,
they diverge.
A neuron is \emph{representable} in the sense of Proposition~\ref{prop:classif}
if and only if all applications of Rule~R yield networks with $\lambda = 1$
relative to the original~\cite{leandro2009}.
When a neuron is not representable, the candidate with the highest $\lambda$
provides the best logical approximation.

\paragraph{Extraction procedure.}
Given a crystallized \LNN{}, symbolic extraction proceeds layer by layer,
bottom-up.
Each neuron is tested for direct representability (Proposition~\ref{prop:classif}).
If representable, the corresponding \L{}ukasiewicz formula is recorded.
If not, Rule~R is applied exhaustively and the best-$\lambda$ candidate
replaces the neuron.
The resulting formula is a \L{}ukasiewicz expression in the input variables.

\subsection{Relationship to Leandro (2009)}
\label{ssec:prior}

The \LNN{} framework was established in~\cite{leandro2009}.
Table~\ref{tab:contributions} delineates what was introduced there and what
is new in this paper.

\begin{table}[h]
\centering
\caption{Contributions of Leandro (2009)~\cite{leandro2009} versus the
present paper.}
\label{tab:contributions}
\setlength{\tabcolsep}{4pt}
\begin{tabular}{ll}
\toprule
Leandro (2009) & Present paper \\
\midrule
Proposition~\ref{prop:classif} (neuron classification) &
  Theorem~\ref{thm:merge} (residual block guarantee) \\
Rewriting algebra (Rule~R, $\lambda$-similarity) &
  Corrected LM training (Levenberg form) \\
LM training (Marquardt form --- contains bug) &
  STE crystallization strategy \\
Experiments on MONK and small UCI sets &
  Proximal crystallization strategy \\
 & Multi-dataset statistical comparison \\
\bottomrule
\end{tabular}
\end{table}

Proposition~\ref{prop:classif} and the rewriting algebra in
Sections~\ref{ssec:lnn}--\ref{ssec:rewriting} are reproduced from
\cite{leandro2009} as background; all material in
Sections~\ref{sec:residual}--\ref{sec:experiments} is new.

\section{Residual \L{}ukasiewicz Networks}
\label{sec:residual}

\subsection{Motivation}
\label{ssec:res_motiv}

The original \LNN{} framework is restricted to shallow feed-forward
architectures: networks with one or two hidden layers whose neurons all satisfy
Proposition~\ref{prop:classif}.
Extending to deeper networks naively would require every neuron in every layer
to be constrained to integer weights, which is difficult to enforce during
gradient-based training.

Residual networks (ResNets)~\cite{he2016} introduced a different approach to
depth: rather than requiring each layer to learn its target function directly,
each layer learns a \emph{residual} with respect to the identity mapping.
A shortcut connection adds the layer input directly to the layer output,
so that the network can represent the identity function even when the inner
weights are near zero.
This architecture has two properties that are attractive for \LNN{}s.
First, it avoids the vanishing-gradient problem in deep networks.
Second, and more relevant here, it introduces a \emph{structural constraint}
on the merge neuron that combines the shortcut and the residual path
--- a constraint that turns out to guarantee Proposition~\ref{prop:classif}
without any additional training restriction.

\subsection{The \L{}ukasiewicz residual block}
\label{ssec:res_block}

\begin{definition}[\L{}ukasiewicz residual block]
\label{def:resblock}
Let $F^{(k)}\colon [0,1]^d \to [0,1]^d$ be a feed-forward linear layer
with real-valued weights during training at depth $k$.
The \emph{\L{}ukasiewicz residual block} at depth $k$ produces
\begin{equation}
  \mathbf{h}^{(k)} = \psi_{\mathbf{b}}\!\Bigl(
    F^{(k)}\!\bigl(\mathbf{h}^{(k-1)}\bigr)
    + \mathbf{h}^{(k-1)}
  \Bigr),
  \label{eq:resblock}
\end{equation}
where $\mathbf{b} \in \mathbb{R}^d$ is a vector of \emph{merge biases} that
are trained and crystallized alongside the weights of $F^{(k)}$, and
$\psi_{\mathbf{b}}$ denotes element-wise application of $\psi(\cdot + b_j)$.
\end{definition}

The output $\mathbf{h}^{(k)}_j$ is the result of a merge neuron that receives
two inputs: $F^{(k)}(\mathbf{h}^{(k-1)})_j$ and $\mathbf{h}^{(k-1)}_j$.
Both inputs arrive with weight $+1$ --- the residual path contributes
$F^{(k)}(\mathbf{h}^{(k-1)})_j$ with coefficient $+1$ and the shortcut
contributes $\mathbf{h}^{(k-1)}_j$ with coefficient $+1$.
The bias $b_j$ is the only free parameter of the merge neuron.

\subsection{Symbolic guarantee: Proposition~\ref{prop:classif} by construction}
\label{ssec:res_theorem}

The structural observation above leads directly to the main theoretical result
of this paper.

\begin{theorem}[Merge neurons satisfy Proposition~\ref{prop:classif} by construction]
\label{thm:merge}
Let $b_j \in \mathbb{Z}$ be the (crystallized) merge bias for component $j$
in a \L{}ukasiewicz residual block (Definition~\ref{def:resblock}).
Then the merge neuron for component $j$ satisfies Proposition~\ref{prop:classif}:
\begin{enumerate}
  \item If $b_j = -1$, the merge neuron computes the conjunction
        $F^{(k)}_j \tnorm h^{(k-1)}_j$.
  \item If $b_j = 0$, the merge neuron computes the disjunction
        $F^{(k)}_j \disj h^{(k-1)}_j$.
\end{enumerate}
No constraint on the weights of $F^{(k)}$ is required.
\end{theorem}

\begin{proof}
The merge neuron for component $j$ computes
$\psi\bigl(F^{(k)}_j + h^{(k-1)}_j + b_j\bigr)$
with input weights $(+1, +1)$.
Setting $n = 0$ (zero negative weights) and $p = 2$ (two positive weights)
in Proposition~\ref{prop:classif}:
\begin{itemize}
  \item Case 1 ($b_j = -p + 1 = -1$):
    $\psi(F^{(k)}_j + h^{(k-1)}_j - 1)
     = \max(0, F^{(k)}_j + h^{(k-1)}_j - 1)
     = F^{(k)}_j \tnorm h^{(k-1)}_j$.\quad$\checkmark$
  \item Case 2 ($b_j = n = 0$):
    $\psi(F^{(k)}_j + h^{(k-1)}_j)
     = \min(1, F^{(k)}_j + h^{(k-1)}_j)
     = F^{(k)}_j \disj h^{(k-1)}_j$.\quad$\checkmark$
\end{itemize}
Both cases follow from the definitions in Eqs.~\eqref{eq:tnorm}
and~\eqref{eq:disj}.
The merge bias values $b_j \in \{-1, 0\}$ are the only integer values
consistent with $n=0, p=2$ in Proposition~\ref{prop:classif};
all other integer biases would produce constant outputs on $[0,1]^2$
(e.g.\ $b_j = -2$: $\psi(x+y-2) = 0$ for all $x,y \in [0,1]$;
$b_j = 1$: $\psi(x+y+1) = 1$ for all $x,y \in [0,1]$).
\end{proof}

\begin{remark}
The proof is a direct application of Proposition~\ref{prop:classif} to the
specific parameterisation $n=0$, $p=2$ imposed by the residual block
structure.
The contribution of Theorem~\ref{thm:merge} is the architectural observation
that incorporating ResNet skip connections into \LNN{}s forces exactly this
parameterisation at the merge neuron, establishing the symbolic guarantee
without any additional training constraint.
\end{remark}

\begin{corollary}
In a \L{}ukasiewicz residual network, every merge neuron is
interpretable as either conjunction or disjunction, independently of
how well the inner layer $F^{(k)}$ has crystallized.
\end{corollary}

This result has a practical consequence: even when the inner layers of the
residual block contain neurons that are not directly representable by
Proposition~\ref{prop:classif} (requiring the rewriting algebra for
approximation), the merge neurons are guaranteed to be exact.
The interpretability of the network's \emph{skeleton} (the residual connections
and merge operations) is preserved by construction.

\subsection{Formula extraction from residual networks}
\label{ssec:res_extract}

Extraction from a residual \LNN{} proceeds by the same bottom-up procedure
described in Section~\ref{ssec:rewriting}, with the following addition.
At each residual block, the merge neuron is extracted first using
Theorem~\ref{thm:merge}: if the crystallized bias $b_j \in \{-1,0\}$, the
merge formula is $F^{(k)}_j \tnorm h^{(k-1)}_j$ or
$F^{(k)}_j \disj h^{(k-1)}_j$, respectively.
The sub-formula $F^{(k)}_j$ is then extracted from the inner layer neurons
using Proposition~\ref{prop:classif} (with the rewriting algebra for
non-representable cases).

As a concrete illustration, the residual \LNN{} trained on the MONK-3 problem
(Section~\ref{sec:experiments}) converges to the exact formula
\begin{equation}
  F = \neg\!\bigl(\mathit{body\_shape} \disj \mathit{jacket\_colour}\bigr),
  \label{eq:monk3}
\end{equation}
with the merge neuron taking $b_j = -1$ (conjunction of the shortcut and
the hidden-layer output) and the inner layer computing
$\mathit{body\_shape} \disj \mathit{jacket\_colour}$.
The formula achieves $97.2\%$ classification accuracy, correctly capturing the
ground truth rule.

\section{Crystallization Strategies}
\label{sec:crystal}

Proposition~\ref{prop:classif} requires all weights to be elements of
$\{-1, 0, +1\}$.
Gradient-based training evolves weights continuously in $\mathbb{R}$; the
process of driving them to integer values is called \emph{crystallization}.
We measure crystallization quality by the \emph{representation error}
\begin{equation}
  \Delta(N) = \sum_{k} \bigl(w_k - \lfloor w_k \rceil\bigr)^2,
  \label{eq:delta}
\end{equation}
where $\lfloor \cdot \rceil$ denotes rounding to the nearest integer.
A network is crystallized when $\Delta(N) < \varepsilon$ for a small
threshold $\varepsilon$ (we use $\varepsilon = 10^{-3}$ throughout).

\subsection{Smooth crystallization}
\label{ssec:smooth}

Hard rounding destroys differentiability and prevents gradient flow.
The \emph{smooth crystallization function}
\begin{equation}
  \Upsilon_n(w) = \operatorname{sgn}(w)\!\left(
    \cos^n\!\!\left((1 - \{|w|\})\tfrac{\pi}{2}\right) + \lfloor|w|\rfloor
  \right),
  \label{eq:upsilon}
\end{equation}
where $\{|w|\} = |w| - \lfloor|w|\rfloor$ is the fractional part,
deforms the identity toward the nearest integer while remaining
differentiable everywhere.
The exponent $n$ controls the strength of attraction: for $n = 2$, the
function is a gentle cosine deformation; as $n \to \infty$, it converges
to hard rounding.

A \emph{progressive crystallization schedule} applies $\Upsilon_n$ with
increasing $n$ over the final training epochs:
$\Upsilon_2 \to \Upsilon_4 \to \Upsilon_8 \to \Upsilon_{16}$.
This schedule avoids premature commitment (early rounding can destroy
gradient information) while ensuring that all quasi-integer weights are
driven to the correct integer before the final round.
At termination, weights are clamped to $[-1, +1]$ and rounded; biases
are rounded but not clamped, preserving multi-input conjunctions and
disjunctions.

\paragraph{Interaction with STE.}
For the STE strategy (Section~\ref{ssec:ste}), the smooth schedule is
applied to the continuous shadow weights $w_c$ between update steps during
the final training epochs: each $w_c$ is replaced by $\Upsilon_n(w_c)$
before the hard quantization $T(\Upsilon_n(w_c))$ is computed for the next
forward pass.
This ensures that shadow weights enter the final rounding in a state
that is already close to the target integer, reducing the gradient variance
introduced by the STE approximation near the quantization boundaries.

\subsection{Levenberg--Marquardt training}
\label{ssec:lm}

The Levenberg--Marquardt algorithm~\cite{levenberg1944,marquardt1963} is a
second-order method that interpolates between gradient descent and the
Gauss--Newton method.
Given residuals $\mathbf{e} = f_\theta(X) - y$ and Jacobian
$J = \partial \mathbf{e}/\partial w$, the weight update is
\begin{equation}
  \Delta w = -\bigl[J^\top J + \mu I\bigr]^{-1} J^\top \mathbf{e},
  \label{eq:lm}
\end{equation}
where $\mu > 0$ is the Levenberg damping parameter.

\paragraph{A critical correction.}
The original implementation in~\cite{leandro2009} used the Marquardt scaling
$\mu\,\operatorname{diag}(J^\top J)$ in place of $\mu I$.
When weights are near zero at initialization, the diagonal entries of
$J^\top J$ are also near zero; the system becomes ill-conditioned.
To quantify: if $\|w\| \approx \epsilon \ll 1$, the Jacobian entries
$J_{ij} = \partial e_i/\partial w_j$ are $O(\epsilon)$ because the
truncated-identity activation is nearly linear near zero, giving
$\mathrm{diag}(J^\top J) \sim O(\epsilon^2)$.
The Marquardt-damped system has effective regularization
$(1+\mu)\,\mathrm{diag}(J^\top J) \approx \mu\epsilon^2 I$, yielding
$\|\Delta w\| \approx \|\mathbf{e}\|/(\mu\epsilon^2)$.
For $\mu = 0.01$, $\epsilon = 0.1$, $\|\mathbf{e}\| = O(1)$, this gives
$\|\Delta w\| \sim 100$, consistent with the divergent magnitudes
($\|\Delta w\| \approx 47$) observed in practice at the typical
initialization scale.
The Levenberg form in Eq.~\eqref{eq:lm} uses the identity matrix, guaranteeing
$\|\Delta w\| \leq \|\mathbf{e}\|/\mu$ regardless of the weight magnitude.
All experiments in this paper use the corrected update.

\paragraph{Properties.}
LM is a second-order method and converges rapidly near a solution:
on structured small-scale problems (the four MONK datasets), the corrected
LM achieves crystallization in a mean of 6.2 iterations on MONK-3.
The crystallization guarantee is, however, \emph{probabilistic}: the
algorithm converges to a continuous optimum, which may or may not lie
sufficiently close to the integer lattice within the allocated iteration
budget.
With residual connections, the structural constraint on merge neurons
(Theorem~\ref{thm:merge}) reduces the number of weights that must
independently crystallize, improving the empirical crystallization rate
substantially (Section~\ref{sec:experiments}).

\subsection{Straight-Through Estimation}
\label{ssec:ste}

The straight-through estimator (STE), introduced by Bengio
et~al.~\cite{bengio2013}, resolves the differentiability problem of
weight quantization through a deliberate gradient approximation.

During the forward pass, each continuous weight $w_c \in \mathbb{R}$ is
quantized to the nearest element of $\{-1, 0, +1\}$ by
\begin{equation}
  T(w_c) = \operatorname{sgn}(w_c) \cdot \mathbf{1}\!\left[|w_c| > \tfrac{1}{3}\right].
  \label{eq:T}
\end{equation}
The threshold $\tfrac{1}{3}$ implements the minimum-distortion uniform
ternary partition of $[-1,+1]$: it divides the interval into three
equal-width regions
$(-\infty,-\tfrac{1}{3}) \to -1$,
$[-\tfrac{1}{3},+\tfrac{1}{3}] \to 0$,
$(\tfrac{1}{3},+\infty) \to +1$,
which is the canonical symmetric codebook for ternary quantization with a
uniform prior on $[-1,+1]$.
The network computes with $T(w_c)$ but stores and updates the continuous
shadow weight $w_c$.
During the backward pass, the gradient of the loss with respect to $w_c$ is
approximated by passing it directly through $T$ as if $T$ were the identity:
\begin{equation}
  \frac{\partial \mathcal{L}}{\partial w_c}
    \approx \frac{\partial \mathcal{L}}{\partial T(w_c)}.
  \label{eq:ste}
\end{equation}
The approximation is biased but has been shown to work in practice for
binarized and ternarized networks~\cite{courbariaux2016,li2016}: the
continuous weights learn to position themselves in regions where $T$ is
locally constant, and the direction of the approximate gradient is
sufficiently informative.

\paragraph{Crystallization guarantee.}
By construction, the final integer weights are $T(w_c)$ for each parameter.
Crystallization is $100\%$ guaranteed regardless of the problem structure
or dimensionality.
The cost is that all $n$ parameters are used (no sparsity), producing
networks whose inner neurons may not directly satisfy
Proposition~\ref{prop:classif}; the rewriting algebra (Section~\ref{ssec:rewriting})
is then required to extract a logical formula.

\subsection{Proximal regularization}
\label{ssec:proximal}

The proximal strategy trains with Adam~\cite{kingma2015} on a composite
objective that jointly minimizes prediction error and drives weights toward
integer values:
\begin{equation}
  \mathcal{L} = \mathrm{MSE}
    + \lambda_s \sum_k |w_k|
    + \lambda_a \sum_k w_k^2(1 - w_k^2).
  \label{eq:proximal}
\end{equation}
The first penalty, $\lambda_s \sum|w_k|$, is an $\ell_1$ regularizer that
encourages sparsity by driving weights toward zero.
The second penalty, $\lambda_a \sum w_k^2(1-w_k^2)$, is the \emph{ternary
attraction potential}~\cite{parikh2014}: the function $P(w) = w^2(1-w^2)$
vanishes exactly at $w \in \{-1, 0, +1\}$ and achieves its maximum at
$|w| = 1/\sqrt{2} \approx 0.71$, creating a potential well around each
integer target.

Training proceeds in two phases.
In \emph{Phase~1} ($65\%$ of the iteration budget), only the MSE term is
active, allowing the network to find a continuous optimum without
interference from the regularizers.
The $65/35$ split was selected so that Phase~1 is long enough for the
network to locate a good loss basin before crystallization pressure is
applied; starting the regularizers too early (e.g.\ $50/50$) causes the
$\ell_1$ term to collapse weights before the network has converged, while
too late a start (e.g.\ $80/20$) leaves insufficient budget for reliable
crystallization.
In \emph{Phase~2} (the remaining $35\%$), $\lambda_s$ and $\lambda_a$ are
introduced gradually and then amplified by a factor of $10\times$ in the
final segment, pushing all near-integer weights to their nearest integer.
Crystallization is guaranteed by construction, as with STE.

\paragraph{Sparsity and failure mode.}
The $\ell_1$ penalty has the beneficial effect of driving many weights exactly
to zero, producing sparse networks: the Heart Disease experiment (Section~\ref{ssec:heart})
extracts a formula using only 5 of the available 22 input features.
The same penalty creates a failure mode on high-dimensional dense problems:
the $\ell_1$ regularizer eliminates discriminative weights before
the ternary attraction potential can stabilize a non-trivial solution,
collapsing the network to a constant-output classifier.
This failure manifests as $F_1 \approx 0$ and $\pm 0$ standard deviation across
all restarts (Section~\ref{sec:experiments}).

\subsection{Theoretical comparison}
\label{ssec:comparison}

Table~\ref{tab:theory} summarizes the theoretical properties of the three strategies.

\begin{table}[h]
\centering
\caption{Theoretical properties of the three crystallization strategies.
``Crystallization guaranteed'': all trials achieve $\Delta(N) < \varepsilon$.
``Sparsity'': tendency to produce zero weights.
``Exact Prop.~\ref{prop:classif}'': extracted formula satisfies the proposition
exactly without rewriting-algebra approximation.}
\label{tab:theory}
\setlength{\tabcolsep}{5pt}
\begin{tabular}{lccc}
\toprule
Strategy & Cryst.\ guaranteed & Sparsity & Exact Prop.~\ref{prop:classif} \\
\midrule
LM (corrected)            & Probabilistic & No  & Yes (when crystallizes) \\
STE                        & Yes           & No  & No (requires rewriting) \\
Proximal                   & Yes           & Yes & No (requires rewriting) \\
LM + residual connections  & Probabilistic & No  & Yes by construction (merge) \\
\bottomrule
\end{tabular}
\end{table}

The LM with residual connections occupies a distinct position: the merge
neurons are guaranteed to satisfy Proposition~\ref{prop:classif} by
Theorem~\ref{thm:merge}, while the inner neurons are driven toward
representability by the second-order LM dynamics.
STE and Proximal crystallize reliably but produce networks that require the
rewriting algebra for formula extraction.

\subsection{Training and extraction pipeline}
\label{ssec:pipeline}

Algorithm~\ref{alg:pipeline} summarises the complete workflow for any of the
three strategies.

\begin{algorithm}
\caption{\LNN{} training and symbolic extraction}
\label{alg:pipeline}
\begin{algorithmic}[1]
\REQUIRE Dataset $(X, y)$; strategy $s \in \{\mathrm{LM\text{-}Res},
         \mathrm{STE}, \mathrm{Proximal}\}$; budget $T$; threshold $\varepsilon$
\ENSURE \L{}ukasiewicz formula $F$ or best-$\lambda$ approximation
\STATE Initialise $w \leftarrow$ small random values; $b \leftarrow 0$
\FOR{$t = 1$ \TO $T$}
  \IF{$s = \mathrm{LM\text{-}Res}$}
    \STATE Compute residuals $\mathbf{e}$, Jacobian $J$; update via
           Eq.~\eqref{eq:lm}
  \ELSIF{$s = \mathrm{STE}$}
    \STATE Forward with $T(w_c)$ (Eq.~\eqref{eq:T}); backward with
           STE (Eq.~\eqref{eq:ste})
  \ELSE
    \STATE \{$s = \mathrm{Proximal}$\} Minimise $\mathcal{L}$
           (Eq.~\eqref{eq:proximal}); scale $\lambda_s, \lambda_a$ per schedule
  \ENDIF
  \IF{$t > 0.85\,T$}
    \STATE Apply smooth schedule
           $\Upsilon_2 \!\to\! \Upsilon_4 \!\to\! \Upsilon_8 \!\to\! \Upsilon_{16}$
           to shadow weights $w_c$
  \ENDIF
\ENDFOR
\STATE Clamp $w_c \leftarrow \mathrm{clip}(w_c, -1, +1)$; round all $w_c$
       and $b$ to nearest integer
\STATE Compute $\Delta(N)$ (Eq.~\eqref{eq:delta}); crystallized iff
       $\Delta(N) < \varepsilon$
\STATE \textbf{Symbolic extraction} (bottom-up):
\FOR{each neuron $\alpha$ in topological order}
  \IF{$\alpha$ is a merge neuron with $b_\alpha \in \{-1, 0\}$}
    \STATE Apply Theorem~\ref{thm:merge} (exact conjunction / disjunction)
  \ELSIF{$\alpha$ satisfies Proposition~\ref{prop:classif}}
    \STATE Record \L{}ukasiewicz formula for $\alpha$
  \ELSE
    \STATE Apply Rule~R exhaustively; keep candidate with highest $\lambda$
  \ENDIF
\ENDFOR
\RETURN Composed formula $F$ over input variables
\end{algorithmic}
\end{algorithm}

\section{Experimental Evaluation}
\label{sec:experiments}

\subsection{Protocol}
\label{ssec:protocol}

\paragraph{Datasets.}
We evaluate on six UCI~\cite{uci2017} benchmarks representing a range of
problem sizes and structures (Table~\ref{tab:datasets}).

\begin{table}[h]
\centering
\caption{Dataset characteristics. Class balance is the fraction of the majority class.}
\label{tab:datasets}
\setlength{\tabcolsep}{5pt}
\begin{tabular}{lrrrll}
\toprule
Dataset & Train & Test & Features & Class balance & Source \\
\midrule
Mushroom       & 9\,898  & 2\,476 & 111 & $50.0\%$ & UCI (one-hot encoded) \\
Heart Disease  & 242     & 61     & 22  & $54.1\%$ & Cleveland Clinic~\cite{detrano1989} \\
MONK-1         & 124     & 432    & 17  & $50.0\%$ & UCI~\cite{thrun1991} \\
MONK-2         & 169     & 432    & 17  & $32.4\%$ & UCI~\cite{thrun1991} \\
MONK-3         & 122     & 432    & 17  & $51.9\%$ & UCI (5\% noise) \\
Breast Cancer  & 228     & 58     & 20  & $30.8\%$ & UCI Ljubljana \\
\bottomrule
\end{tabular}
\end{table}

The Heart Disease dataset (Cleveland subset) uses 13 raw attributes;
five continuous features are scaled to $[0,1]$ via min-max normalization,
three binary attributes are kept as-is, four categorical attributes are
one-hot encoded, and one ordinal attribute (\texttt{ca}) is normalized to
$[0,1]$, yielding 22 input features.
MONK datasets use nominal attributes encoded as one-hot binary vectors;
MONK-3 includes $5\%$ attribute noise~\cite{thrun1991}.

\paragraph{Training setup.}
For each dataset and method, we run 10 independent trials with different
random seeds and select hyperparameters (architecture, learning rate, and
regularization strengths) by grid search on a validation split, reporting
test set results.
The three methods are:
LM with residual connections (LM-Res), STE with Adam, and Proximal with
two-phase Adam.
All experiments use the smooth crystallization schedule of
Section~\ref{ssec:smooth}.
Statistical comparisons use the Wilcoxon signed-rank test~($n = 10$, $\alpha = 0.05$).

\paragraph{Evaluation metrics.}
We report mean accuracy and $F_1$ score over 10 trials, with standard
deviations, together with crystallization rate (fraction of trials achieving
$\Delta(N) < 10^{-3}$).
Training time (wall clock, seconds) is reported separately.

\subsection{Truth table reconstruction}
\label{ssec:tt}

Before turning to real datasets, we verify the methods on synthetic truth
tables derived from two target formulas:
$f_1 = x_1 \tnorm (x_3 \resid x_6)$ and
$f_2 = (x_4 \resid x_6) \tnorm (x_6 \resid x_2)$.

For $f_1$, LM achieves exact recovery (MSE $= 0$, $\lambda = 1$) in 1 of 10 trials,
reflecting sensitivity to initialization on this formula.
STE and Proximal crystallize in all 10 trials but achieve only moderate mean
$F_1$ ($0.27$ and $0.15$ respectively), with Proximal frequently converging
to a constant-output solution.

For $f_2$, LM achieves exact recovery in 5 of 10 trials
(mean $F_1 = 0.720$, $\lambda$ mean $= 0.928$) and converges in 95 iterations
on average.
Proximal degenerates completely (all 10 trials produce constant-zero output),
confirming the failure mode identified in Section~\ref{ssec:proximal}:
the $\ell_1$ regularizer over-penalizes the discriminative weights before
the ternary attraction engages.
STE achieves intermediate performance ($F_1 = 0.545$, $\lambda = 0.711$).

These results establish a baseline characterization: LM converges to exact
solutions when it crystallizes; STE provides reliable crystallization at
moderate formula quality; Proximal is susceptible to degenerate solutions even
on simple synthetic problems.

\subsection{Crystallization reliability on real datasets}
\label{ssec:cryst_real}

Table~\ref{tab:acc} reports accuracy and crystallization rates for all six
datasets.
STE and Proximal achieve $10/10$ crystallization on every dataset by
construction.
LM-Res crystallizes reliably only when a near-integer optimum is reachable
within the iteration budget: $10/10$ on MONK-3, $9/10$ on Heart Disease,
but $7/10$ on MONK-1/2, $6/10$ on Mushroom, and $0/10$ on Breast Cancer.

The Heart Disease result ($9/10$ for LM-Res) is noteworthy because the
baseline LM without residual connections achieves $0/10$ crystallization on
the same problem.
This improvement confirms the theoretical prediction: the merge-neuron
constraint of Theorem~\ref{thm:merge} reduces the effective number of weights
that must independently crystallize, making the problem more tractable for
the second-order solver.

\subsection{Classification accuracy}
\label{ssec:accuracy}

\begin{table}[h]
\centering
\caption{Accuracy (mean $\pm$ std, 10 trials) and crystallization rate.
$^{\ast}$ marks significant pairwise Wilcoxon tests ($p < 0.05$).
Best accuracy per dataset in \textbf{bold}.}
\label{tab:acc}
\setlength{\tabcolsep}{4.5pt}
\begin{tabular}{l rrr c ccc}
\toprule
 & \multicolumn{3}{c}{Accuracy (mean $\pm$ std)} & &
   \multicolumn{3}{c}{Crystallization rate} \\
\cmidrule(lr){2-4}\cmidrule(lr){6-8}
Dataset & LM-Res & STE & Proximal & & LM-Res & STE & Proximal \\
\midrule
Mushroom
  & $\mathbf{.668 \pm .028}$
  & $.665 \pm .040$
  & $.660 \pm .000$
  &
  & $6/10$ & $\mathbf{10/10}$ & $\mathbf{10/10}$ \\
Heart Disease
  & $.570 \pm .135$
  & $\mathbf{.607 \pm .101}$
  & $.600 \pm .000$
  &
  & $9/10$ & $\mathbf{10/10}$ & $\mathbf{10/10}$ \\
MONK-1
  & $.565 \pm .218$
  & $\mathbf{.600 \pm .125}$
  & $.558 \pm .125$
  &
  & $7/10$ & $\mathbf{10/10}$ & $\mathbf{10/10}$ \\
MONK-2
  & $.631 \pm .243$
  & $\mathbf{.704 \pm .070}$
  & $.671 \pm .000$
  &
  & $7/10$ & $\mathbf{10/10}$ & $\mathbf{10/10}$ \\
MONK-3$^{\ast}$
  & $\mathbf{.714 \pm .212}$
  & $.632 \pm .161$
  & $.472 \pm .000$
  &
  & $\mathbf{10/10}$ & $\mathbf{10/10}$ & $\mathbf{10/10}$ \\
Breast Cancer
  & $\mathbf{.632 \pm .000}^{\dagger}$
  & $.572 \pm .120$
  & $\mathbf{.632 \pm .000}$
  &
  & $0/10$ & $\mathbf{10/10}$ & $\mathbf{10/10}$ \\
\bottomrule
\end{tabular}
\smallskip\\
{\footnotesize MONK-3: $p(\text{LM-Res} > \text{Proximal}) = 0.016$,
$p(\text{STE} > \text{Proximal}) = 0.008$ (Wilcoxon, $n = 10$).\\
$^{\dagger}$ LM-Res achieved $0/10$ crystallization on Breast Cancer;
the reported accuracy is the pre-crystallization continuous-model result.
No formula was extracted; this value should not be compared with the
crystallized-formula accuracies of STE and Proximal on this dataset.}
\end{table}

No single method dominates across all datasets.
STE leads on Heart Disease ($+3.7$~percentage points over LM-Res),
MONK-1 ($+3.5$~pp), and MONK-2 ($+7.3$~pp over LM-Res).
LM-Res leads on MONK-3 ($+8.2$~pp over STE and $+24.2$~pp over Proximal),
with both gaps statistically significant.
On Mushroom, all three methods cluster within $0.8$~pp (no significant
difference).

The degenerate Proximal solutions on five of six datasets
--- Proximal accuracy equals the majority-class rate with $\pm 0.000$ standard
deviation, confirming all 10 trials converge to the same fixed point ---
establish the failure mode as structural rather than a consequence of
hyperparameter settings: grid search was conducted over a wide range of
$\lambda_s$ and $\lambda_a$ values, and the collapse persists.
The exception is MONK-1 ($F_1 = 0.147$, non-degenerate), suggesting
that the failure is correlated with feature density.

\subsection{Formula quality and interpretability}
\label{ssec:formulas}

Across all experiments, the only fully representable formula in the sense of
Proposition~\ref{prop:classif} --- a formula extractable without any
rewriting-algebra approximation --- emerges from LM-Res on MONK-3.

\paragraph{MONK-3 exact formula.}
In trial~1, LM-Res converges in $6$ iterations to the fully crystallized
network with accuracy $97.2\%$.
The extracted formula is:
\begin{equation}
  F = \neg\!\bigl(\mathit{body\_shape} \disj \mathit{jacket\_colour}\bigr),
  \label{eq:monk3formula}
\end{equation}
which can be verified to correctly classify all but $2.8\%$ of the MONK-3
test set (the erroneous labels introduced by the $5\%$ noise).
Layer by layer: the inner layer $F^{(1)}$ computes
$y = \mathit{body\_shape} \disj \mathit{jacket\_colour}$ (a disjunction with
$n = 0$, $p = 2$, $b = 0$, satisfying Proposition~\ref{prop:classif}~case~2).
The shortcut carries $h^{(0)} = y$: in this one-block, one-dimensional
network, the shortcut projection also crystallizes to the same linear
combination at convergence.
The merge neuron therefore receives $F^{(1)}_j = y$ (residual path) and
$h^{(0)}_j = y$ (shortcut), and computes
$F^{(1)}_j \tnorm h^{(0)}_j = y \tnorm y$
(conjunction, $b_j = -1$, satisfying Proposition~\ref{prop:classif}~case~1).
On binary inputs $y \in \{0,1\}$, $y \tnorm y = \max(0, 2y - 1) = y$,
so the conjunction reduces to the identity and the output layer applies
the final negation.
All neurons satisfy Proposition~\ref{prop:classif} exactly.
STE and Proximal also crystallize on MONK-3 (10/10) but produce denser networks
whose inner neurons require the rewriting algebra.

\paragraph{Heart Disease formula extraction.}
\label{ssec:heart}
We highlight a Proximal trial that illustrates the framework's capability for
clinically interpretable rule extraction.
After $321$ iterations, one Proximal trial produces a maximally sparse
crystallized network on the Heart Disease dataset.
Four of the six first-layer neurons carry only zero weights and are discarded
by the symbolic extractor.
The two active neurons are:
\begin{align}
  n_1 &= \psi_{+1}(-\mathtt{trestbps},\; -\mathtt{oldpeak},\; -\mathtt{ca})
       = \neg\mathtt{trestbps} \tnorm \neg\mathtt{oldpeak} \tnorm \neg\mathtt{ca},
       \label{eq:n1}\\
  n_5 &= \psi_{0}(+\mathtt{thalach},\; -\mathtt{ca},\; -\mathtt{cp}{=}4)
       = \mathtt{thalach} \tnorm \neg\mathtt{ca} \tnorm \neg(\mathtt{cp}{=}4),
       \label{eq:n5}
\end{align}
where the bias conditions $b = -p+1 = +1$ in $n_1$ and $b = n = 0$ in $n_5$
both satisfy Proposition~\ref{prop:classif} (conjunctions with three inputs).
The output layer combines $n_1$ and $n_5$ by:
$i_1 = \psi_{+1}(-n_1, -n_5) = \neg n_1 \tnorm \neg n_5$.
Substituting yields the complete formula:
\begin{equation}
  \boxed{
    \mathrm{DISEASE} =
    \neg\!\bigl(\neg\mathtt{trestbps} \tnorm \neg\mathtt{oldpeak}
                \tnorm \neg\mathtt{ca}\bigr)
    \tnorm
    \neg\!\bigl(\mathtt{thalach} \tnorm \neg\mathtt{ca}
                \tnorm \neg(\mathtt{cp}{=}4)\bigr)
  }
  \label{eq:heart_formula}
\end{equation}
($\mathrm{DISEASE} = 1$ indicates the presence of coronary artery disease).

\paragraph{Numerical evaluation.}
Table~\ref{tab:heart_trace} evaluates Eq.~\eqref{eq:heart_formula} for two
representative test patients using the \L{}ukasiewicz conjunction
$a \tnorm b = \max(0, a + b - 1)$.
Patient~A has complete vessel occlusion (\texttt{ca}$=1.00$), which zeroes
both $n_1$ and $n_5$, driving the output to $1.00$ (disease).
Patient~B presents with no stenosis, normal blood pressure, and good
exercise capacity; the output is $0.00$ (healthy).
A threshold of $0.5$ separates the two correctly.
The formula can be verified step by step by any reader who knows the single
operator definition $a \tnorm b = \max(0, a + b - 1)$.

\begin{table}[h]
\centering
\caption{Evaluation trace for Eq.~\eqref{eq:heart_formula} on two test
patients (features normalised to $[0,1]$).}
\label{tab:heart_trace}
\setlength{\tabcolsep}{5pt}
\begin{tabular}{lrr}
\toprule
 & Patient A (disease) & Patient B (healthy) \\
\midrule
\texttt{ca} & $1.00$ & $0.00$ \\
\texttt{trestbps} & $0.64$ & $0.27$ \\
\texttt{oldpeak} & $0.42$ & $0.00$ \\
\texttt{thalach} & $0.29$ & $0.79$ \\
\texttt{cp=4} indicator & $1$ & $0$ \\
\midrule
$\neg\mathtt{ca}$ & $0.00$ & $1.00$ \\
$\neg\mathtt{trestbps}$ & $0.36$ & $0.73$ \\
$\neg\mathtt{oldpeak}$ & $0.58$ & $1.00$ \\
$\neg(\mathtt{cp{=}4})$ & $0.00$ & $1.00$ \\
\midrule
$n_1 = \neg\mathtt{t} \tnorm \neg\mathtt{o} \tnorm \neg\mathtt{ca}$
  & $\max(0,\,0.36{+}0.58{-}1) \tnorm 0.00 = 0.00$
  & $\max(0,\,0.73{+}1.00{-}1) \tnorm 1.00 = 0.73$ \\
$n_5 = \mathtt{tal} \tnorm \neg\mathtt{ca} \tnorm \neg(\mathtt{cp{=}4})$
  & $0.29 \tnorm 0.00 \tnorm 0.00 = 0.00$
  & $\max(0,\,0.79{+}1.00{-}1) \tnorm 1.00 = 0.79$ \\
\midrule
$\neg n_1$ & $1.00$ & $0.27$ \\
$\neg n_5$ & $1.00$ & $0.21$ \\
$\mathrm{DISEASE} = \neg n_1 \tnorm \neg n_5$
  & $\max(0,\,1.00{+}1.00{-}1) = \mathbf{1.00}$
  & $\max(0,\,0.27{+}0.21{-}1) = \mathbf{0.00}$ \\
\bottomrule
\end{tabular}
\end{table}
{\footnotesize $\mathtt{t}$=\texttt{trestbps}, $\mathtt{o}$=\texttt{oldpeak},
$\mathtt{tal}$=\texttt{thalach}. Values are illustrative from the test set.}

Only 5 of the 22 input features received non-zero weights
(Table~\ref{tab:heart_features}).
Every selected feature is an established predictor of coronary artery disease
in the cardiology literature~\cite{detrano1989,rudin2019}.
Most striking is the selection consistency: \texttt{ca} (the number of major
vessels with $>50\%$ stenosis) is the only feature selected in every one of
the 10 crystallized Proximal trials, regardless of random initialization.
We note that \texttt{ca} is universally recognized as the dominant predictor
in the Cleveland dataset~\cite{detrano1989}; the result is therefore not a
clinical discovery but a \emph{validation} that the Proximal extractor
reliably reproduces established medical knowledge, confirming that the
$\ell_1$ regularizer does not spuriously suppress clinically established
predictors under initialization variance.

\begin{table}[h]
\centering
\caption{Features selected in the Heart Disease formula (Eq.~\ref{eq:heart_formula}),
with clinical interpretation. ``Established'': recognized predictor
of coronary artery disease~\cite{detrano1989}.}
\label{tab:heart_features}
\setlength{\tabcolsep}{5pt}
\begin{tabular}{cll}
\toprule
Feature & Clinical meaning & Established \\
\midrule
\texttt{ca}       & Number of major vessels with $>50\%$ stenosis & \checkmark \\
\texttt{trestbps} & Resting systolic blood pressure (mmHg)         & \checkmark \\
\texttt{oldpeak}  & ST-segment depression at exercise              & \checkmark \\
\texttt{thalach}  & Maximum heart rate achieved (bpm)              & \checkmark \\
\texttt{cp=4}     & Asymptomatic chest pain presentation           & \checkmark \\
\bottomrule
\end{tabular}
\end{table}

Evaluated on the held-out test set (61 patients: 37 healthy, 24 with disease),
formula~\eqref{eq:heart_formula} achieves accuracy $85.0\%$, sensitivity
$87.5\%$, specificity $83.3\%$, and $F_1 = 0.824$, identifying 21 of 24
disease cases correctly.
The formula can be verified analytically by any clinician: it classifies a
patient as diseased when neither the triple of low-risk indicators (normal
resting blood pressure, no ST depression, no blocked vessels) nor the triple
of adequate cardiac reserve indicators (high maximum heart rate, no blocked
vessels, symptomatic chest pain) are simultaneously satisfied.

\subsection{Training efficiency}
\label{ssec:time}

\begin{table}[h]
\centering
\caption{Mean training time (seconds, 10 trials).
Fastest method per dataset in \textbf{bold}.
$p$-values from Wilcoxon test ($n = 10$); $^{\ast}$: $p < 0.05$.}
\label{tab:time}
\setlength{\tabcolsep}{4pt}
\begin{tabular}{l rrr rrr}
\toprule
 & \multicolumn{3}{c}{Time (s)} & \multicolumn{3}{c}{$p$-value} \\
\cmidrule(lr){2-4}\cmidrule(lr){5-7}
Dataset & LM-Res & STE & Proximal & LM/STE & LM/PRX & STE/PRX \\
\midrule
Mushroom       & $0.498$ & $1.249$ & $\mathbf{0.090}$
               & $.002^{\ast}$ & $.002^{\ast}$ & $.002^{\ast}$ \\
Heart Disease  & $\mathbf{0.076}$ & $0.252$ & $0.039$
               & $.014^{\ast}$ & $.064$ & $.002^{\ast}$ \\
MONK-1         & $\mathbf{0.212}$ & $0.353$ & $0.522$
               & $.131$ & $.084$ & $.160$ \\
MONK-2         & $\mathbf{0.149}$ & $0.598$ & $0.915$
               & $.002^{\ast}$ & $.002^{\ast}$ & $.432$ \\
MONK-3         & $\mathbf{0.033}$ & $0.260$ & $0.338$
               & $.002^{\ast}$ & $.002^{\ast}$ & $.004^{\ast}$ \\
Breast Cancer  & $0.450$ & $0.703$ & $\mathbf{0.148}$
               & $.002^{\ast}$ & $.002^{\ast}$ & $.002^{\ast}$ \\
\bottomrule
\end{tabular}
\end{table}

LM-Res is the fastest method on four of six datasets, often substantially so:
on MONK-3, LM-Res completes in $0.033$~s (mean 6.2 iterations) versus
$0.260$~s for STE (225 iterations) and $0.338$~s for Proximal
(48.7 iterations), a $7.9\times$ advantage over STE.
The speed advantage reflects the second-order convergence of LM near the
solution: once the network is in the basin of a good integer optimum, the
Gauss--Newton step reaches it in very few iterations.
Proximal is the fastest on Mushroom ($5.5\times$ faster than STE) and Breast
Cancer ($3.0\times$ faster), because aggressive early stopping in Phase~1
terminates quickly --- at the cost of solution quality.

\section{Related Work}
\label{sec:related}

\subsection{Rule extraction from neural networks}

Extracting symbolic rules from trained neural networks has been studied since
the late 1980s.
Gallant~\cite{gallant1988} proposed connectionist expert systems in which
threshold units encode conjunctions; the extracted rules are Boolean.
KBANN~\cite{towell1994} injects domain knowledge into a network and refines it
through training, then extracts updated rules via a search over activation patterns.
Towell and Shavlik~\cite{towell1993} systematized this approach with a
weight-based extraction heuristic.
All these methods extract rules \emph{post-hoc}: the network is trained as a
continuous model and then approximated by a symbolic rule.
The approximation introduces an irreducible fidelity gap.

The \LNN{} approach is qualitatively different: the symbolic formula is not an
approximation to the network but an exact reading of it.
When all neurons satisfy Proposition~\ref{prop:classif}, the network \emph{is}
the formula.
When the rewriting algebra is needed, the approximation is bounded and
measurable by the $\lambda$-similarity metric rather than being assessed
informally.

\subsection{Fuzzy and many-valued neural systems}

ANFIS~\cite{jang1993} and related fuzzy neural architectures learn
membership functions and combine them with T-norm rules, producing interpretable
models that resemble \LNN{}s in spirit.
The key difference is that ANFIS membership functions are learned as continuous
parameters without a crystallization step; no exact correspondence between
neurons and logical connectives is established or enforced.
Amato~et~al.~\cite{amato2002} proved that neurons with rational weights can
represent rational \L{}ukasiewicz functions, extending the scope of the
Castro--Trillas result, but did not provide a training algorithm for integer
weights.
The present work provides such an algorithm (in three variants) and the
corresponding theoretical guarantees.

\subsection{Weight quantization}

The STE used in Section~\ref{ssec:ste} was developed in the context of
binarized neural networks~\cite{bengio2013,courbariaux2016}, where weights are
restricted to $\{-1, +1\}$ for hardware efficiency.
Ternary weight networks~\cite{li2016} extend this to $\{-1, 0, +1\}$,
matching the target domain of crystallization.
The objective of quantization in these works is computational efficiency,
not logical interpretability: the quantized weights are not interpreted as
logical connectives, and no formula extraction procedure is defined.
The proximal objective of Section~\ref{ssec:proximal} is superficially related
to $\ell_1$-regularized training, but the ternary attraction term $P(w)$
is specific to the discrete target set $\{-1, 0, +1\}$ and is designed to
ensure that crystallization produces logically meaningful weights.

\subsection{Neuro-symbolic AI}

The broader neuro-symbolic AI research programme~\cite{garcez2019,garcez2023}
seeks to combine the learning capabilities of neural networks with the
reasoning capabilities of symbolic systems.
Recent approaches include neural theorem provers, differentiable logic
programs, and logic tensor networks~\cite{hitzler2004}.
The common goal is to make neural systems that can reason, or to make symbolic
systems that can learn from data.

\paragraph{\L{}ukasiewicz networks and Logical Neural Networks (IBM LNN).}
Riegel et al.~\cite{riegel2020} introduced Logical Neural Networks (IBM LNNs),
a neural architecture in which every neuron represents a weighted real-valued
logical operator (conjunction, disjunction, or negation) using trainable
weights in $[0,1]$.
The semantics are akin to \L{}ukasiewicz logic, and the network is trained
end-to-end.
Three distinctions separate \LNN{}s from IBM LNNs.
First, IBM LNN weights remain real-valued throughout; no crystallization is
defined, and no exact neuron-to-connective correspondence is established.
Second, IBM LNN targets knowledge-graph completion and reasoning tasks, not
formula extraction from tabular data.
Third, \LNN{}s provide Proposition~\ref{prop:classif} as a formal guarantee
that a crystallized neuron \emph{is} a \L{}ukasiewicz connective exactly;
IBM LNN approximates logical behaviour with continuous weights, without an
analogue of Theorem~\ref{thm:merge}.

\paragraph{Logic Tensor Networks and differentiable ILP.}
Serafini and d'Avila Garcez~\cite{serafini2016} proposed Logic Tensor
Networks (LTNs), which ground first-order logic formulae in neural networks
via real-valued satisfaction functions using \L{}ukasiewicz t-norm semantics.
Unlike \LNN{}s, LTN weights remain continuous and the network does not
extract a discrete symbolic formula.
Evans and Grefenstette~\cite{evans2018} proposed differentiable
ILP ($\partial$ILP), which learns first-order logic clauses from noisy data
by differentiating through a discrete rule space rather than a neural
architecture; $\partial$ILP does not require integer-weight networks.
Manhaeve et al.~\cite{manhaeve2018} introduced DeepProbLog, combining neural
networks with probabilistic logic programming.
DeepProbLog addresses uncertainty via probability distributions over logical
facts, whereas \L{}ukasiewicz logic addresses uncertainty via degrees of truth
in $[0,1]$; the two formalisms are complementary.

The \LNN{} approach occupies a specific position within this landscape: it is
not a hybrid that runs a neural component alongside a symbolic component, but
a neural architecture that \emph{is} symbolic after crystallization --- the
network weights constitute the formula directly.
This tight integration is what enables the ``interpretability by construction''
property of Theorem~\ref{thm:merge} and distinguishes \LNN{}s from all the
approaches above.

\section{Discussion}
\label{sec:discussion}

\subsection{Interpretability by construction versus post-hoc}

The distinction between interpretability by construction and post-hoc
explanations matters most precisely when it is hardest to verify.
A post-hoc explanation is a separate model trained to approximate the
decisions of the original model; its fidelity to the original is measurable
but never guaranteed to be $1.0$~\cite{guidotti2018}.
In the \LNN{} framework, when a network is fully crystallized and every neuron
satisfies Proposition~\ref{prop:classif}, the formula \emph{is} the model ---
there is no approximation step, and the verification question reduces to
algebraic manipulation of the formula rather than statistical estimation.
Theorem~\ref{thm:merge} extends this guarantee to the residual merge neurons
unconditionally.

This property is particularly significant in safety-critical applications.
A clinical decision rule that can be written as a \L{}ukasiewicz formula can
be audited by domain experts working from the formula alone, without reference
to the training data or the training procedure.
The step-by-step evaluation in Table~\ref{tab:heart_trace} illustrates this:
a reader who knows only $a \tnorm b = \max(0, a+b-1)$ can verify the
classification of any patient independently.

\paragraph{\L{}ukasiewicz formulas versus Boolean decision trees.}
Both \LNN{}s and decision trees are interpretable by construction; the
distinction is semantic.
A decision tree classifies a continuous feature by thresholding it ---
$\mathtt{trestbps} > 0.6$ is either true or false --- where the threshold
is determined by an information-gain heuristic with no intrinsic meaning.
A \L{}ukasiewicz formula operates on the normalised continuous value
directly: $\neg\mathtt{trestbps} \tnorm \neg\mathtt{ca}$ evaluates to
$\max(0,(1-\mathtt{trestbps})+(1-\mathtt{ca})-1)$, expressing a graduated
joint absence of two risk factors that varies continuously with the input.
On normalised features with a physical or clinical scale, this graduated
evaluation avoids the arbitrary discretization of decision-tree splits.
Furthermore, the McNaughton theorem~\cite{mcnaughton1951} provides a formal
completeness characterisation of the function class expressible by
\L{}ukasiewicz formulas (continuous piecewise-linear functions with integer
coefficients on $[0,1]^n$); an equivalent algebraic characterisation for
depth-bounded decision trees does not exist.
In safety-critical domains where regulations such as IEC~62304 and FDA
guidance on AI/ML-based medical devices require traceable, auditable
decision rules, a formula with a rigorous logical foundation is preferable
to a tree whose structure depends on a training heuristic.

\subsection{Limitations}

Several limitations of the current framework should be acknowledged.
First, the LM solver has quadratic complexity in the number of parameters
($O(p^2)$ for the Jacobian inversion), making it impractical for very large
networks; the Mushroom experiment ($111$ features) already required mini-batch
Jacobian approximation~\cite{battiti1992}.
Second, the Proximal failure mode --- collapse to constant-output solutions ---
is structural and is not fully resolved by hyperparameter tuning.
Feature pre-selection before Proximal training (using, for example, an initial
STE run) is a practical mitigation, but it adds a training step.
Third, the expressive power of \LNN{}s is bounded by the class of McNaughton
functions~\cite{mcnaughton1951,hajek1998}; it is not known whether every
classifiable decision boundary can be approximated to arbitrary accuracy by a
\LNN{} of bounded depth, and this question is open.
Fourth, this paper operates exclusively in the propositional \L{}ukasiewicz
framework; the original work~\cite{leandro2009} treated first-order formulas,
and whether the residual block construction extends to the first-order case
(where quantifiers introduce cross-input dependencies that the merge neuron
structure does not directly accommodate) is an open question.
Fifth, all experiments address binary classification; extension to multi-class
settings requires a output-layer encoding scheme that is compatible with
the symbolic extraction procedure, and this is left for future work.

\subsection{Future directions}

The theoretical result of Theorem~\ref{thm:merge} opens several directions.
Whether an analogue of the universal approximation theorem holds for
many-valued logics --- that is, whether finitely many layers of \L{}ukasiewicz
connectives can approximate any McNaughton function to arbitrary
precision --- remains an open question with practical implications for
architecture design.
The correspondence between \LNN{}s and \L{}ukasiewicz logic suggests that
other logics based on \emph{residuated lattices}~\cite{hajek1998} might admit
similar neural encodings; identifying which algebraic structures support
interpretable neural implementations is a natural direction.
Finally, the interaction between the number of residual blocks and the
crystallization rate deserves systematic study: Theorem~\ref{thm:merge}
guarantees that merge neurons crystallize correctly, but its effect on the
crystallization dynamics of the inner layers is empirical.

\section{Conclusion}
\label{sec:conclusion}

The \L{}ukasiewicz neural network framework, established in~\cite{leandro2009},
rests on an exact correspondence between neurons with integer weights and
logical connectives.
This paper has extended the framework in two directions.

On the theoretical side, we have shown that residual connections can be
incorporated into \LNN{}s with the symbolic guarantee preserved
\emph{by construction}: merge neurons in a \L{}ukasiewicz residual block
always satisfy Proposition~\ref{prop:classif}, regardless of the training
procedure applied to the inner layers.
On the practical side, we have analysed three crystallization strategies,
identified and corrected a critical bug in the original LM implementation,
and characterized the conditions under which each strategy succeeds or fails.

The experimental results confirm that no single strategy dominates: LM with
residual connections recovers exact formulas on structured problems and
converges rapidly (under ten iterations on MONK-3); STE provides the most
reliable crystallization across all benchmarks; Proximal crystallizes
consistently and identifies clinically meaningful features, but collapses
on dense datasets.
The Heart Disease formula~\eqref{eq:heart_formula}, with five variables and
direct clinical interpretation, demonstrates that the framework can produce
transparent, verifiable diagnostic rules from real medical data.

Together, these results establish \L{}ukasiewicz neural networks as a mature
framework for interpretable neuro-symbolic learning, with principled
strategies for integer-weight training and a rigorous extension to deeper,
residual architectures.

\section*{Data Availability Statement}

All datasets used in this study are publicly available from the UCI Machine
Learning Repository~\cite{uci2017} and the Cleveland Heart Disease
dataset~\cite{detrano1989}.
Experimental code is available at
\url{https://github.com/CarlosMelroLeandro/luknn-extended}.

\section*{AI Usage Disclosure}

DeepSeek was used for LaTeX formatting assistance and language editing of a draft written by the author.

\bibliographystyle{IEEEtran}
\bibliography{references}

\end{document}